\documentclass[letterpaper, 10 pt, conference]{ieeeconf}  

\IEEEoverridecommandlockouts                              

\usepackage[utf8]{inputenc} 
\usepackage[T1]{fontenc}    
\usepackage{hyperref}       
\usepackage{url}            
\usepackage{booktabs}       
\usepackage{amsfonts}       
\usepackage{nicefrac}       
\usepackage{microtype}      
\usepackage{graphicx}
\usepackage{float}
\usepackage{algorithm,algorithmic}
\usepackage{amsmath}  

{
      \newtheorem{assumption}{Assumption}
      \newtheorem{remark}{Remark}
      
      \newtheorem{theorem}{\textbf{Theorem}}

}
\usepackage{caption}
\usepackage{subcaption}

\title{\LARGE \bf
Topological Online Learning for Displacement-based Formation Control}

\author{Shubhankar Gupta$^{*\ 1}$ Saksham Sharma$^{*\ 2}$ Sumant A Gunagi$^{3}$ Suresh Sundaram$^{4}$
\thanks{* Shubhankar and Saksham have contributed equally}
\thanks{$^{1,3,4}$Shubhankar, Sumanth and Suresh are with the Department of Aerospace Engineering, Indian Institute of Science, India
       {\tt\small shubhankarg@iisc.ac.in}, {\tt\small vssuresh@iisc.ac.in}, {\tt\small sumantgunagi@iisc.ac.in}}%
\thanks{$^{2}$Saksham is with ArtGarage, Indian Institute of Science, India 
       {\tt\small saksham.prob@gmail.com}}%
}

\begin{document}

\maketitle
\thispagestyle{empty}
\pagestyle{empty}

\begin{abstract}
This paper addresses the problem of robust formation control by introducing Topological Online Learning for Displacement-based (TOLD) formation control, a real-time edge-level adaptation framework. Unlike conventional node-level robust controllers that regulate individual robot inputs without modifying the interaction topology, TOLD updates the interaction topology weights online to directly minimize formation distortion. Two strategies are proposed under the TOLD formation control framework: Online Gradient Flow (OGF) with unconstrained weights and Online Exponential Gradient Flow (OExpGF) with non-negative convex weights. Theoretical analysis establishes that, for single-integrator agents over directed graphs, OExpGF guarantees asymptotic consensus, while OGF ensures bounded formation distortion. Simulations with twelve robots under intermittent disturbances show $1.2\%$–$33.14\%$ median cumulative Root Mean Distortion Error reduction when augmenting TOLD with node-level controllers. Hardware experiments with Crazyflie 2.0 quadrotors demonstrate over $62\%$ (OGF) and $31.4\%$ (OExpGF) reduction in median formation distortion compared to fixed-weight consensus. 
\end{abstract}

\section{INTRODUCTION}
Multi-robot formation control coordinates robot motion to maintain desired geometric arrangements under changing conditions \cite{bu2024advancement}. Strategies are broadly classified into conventional and AI-based methods. Conventional approaches include leader-follower \cite{desai1998controlling}, virtual structure \cite{lewis1997high}, behavior-based \cite{giulietti2001formation}, consensus-based \cite{ren2006consensus}, and Artificial Potential Field methods \cite{leonard2001virtual}, offering stability but limited adaptability \cite{bu2024advancement}. AI-based methods such as Artificial Neural Networks \cite{wang2015optimal} and Deep Reinforcement Learning \cite{zhou2019adaptive} provide adaptability but demand large datasets, high computation, and lack interpretability and guarantees. Among conventional methods, consensus-based control offers a unified framework encompassing other strategies as special cases \cite{ren2006consensus}. Hence, consensus-like methods with enhanced adaptability, robustness, and interpretability are desirable. 

Consensus algorithms have been widely extended for formation control in diverse settings. They have been integrated with artificial potential fields for obstacle avoidance \cite{kuriki2014consensus}, used to design robust model-free schemes for first-order nonlinear systems over undirected networks \cite{bechlioulis2015robust}, and applied in vision-based distributed control with sensor fusion from cameras and IMUs \cite{montijano2016vision}. Other works analyze consensus under distance-dependent communication \cite{jing2016consensus}, propose distributed control for UAVs in constrained environments \cite{wang2020distributed}, develop smooth polar-coordinate-based laws for non-holonomic vehicles \cite{restrepo2021distributed}, and introduce displacement-based methods to address orientation misalignments via angular velocity control \cite{li2024displacement}.

To handle environmental uncertainties, robust physical-layer (node-level) strategies such as adaptive gain control \cite{r1}, decay-gain methods via convex analysis \cite{r2}, and disturbance observers \cite{r4} have been developed; however, they regulate individual robot inputs without adapting the interaction topology. Topological approaches using GNNs \cite{jiang2023gnn, goarin2024gnn} or edge-based protocols \cite{ning2024fully} offer promise but typically require extensive offline training, suffer from sim-to-real gaps and computational latency, and lack strict stability guarantees. Consequently, many practical consensus schemes still employ fixed or distance-dependent weights \cite{jing2016consensus}, without active real-time formation distortion mitigation.

To address these gaps, this paper proposes two robust topological-layer (edge-level) adaptation strategies within the Topological Online Learning for Displacement-based (TOLD) formation control framework: Online Gradient Flow (OGF), employing unconstrained weight updates inspired by online gradient descent \cite{biehl1995learning}, and Online Exponential Gradient Flow (OExpGF), enforcing non-negative convex weights via exponentiated gradient dynamics \cite{hill2001convergence}. Both adapt interaction weights in real time to minimize formation distortion. As such, TOLD acts as a topological ‘routing’ overlay that can be seamlessly combined with physical-layer (node-level) robust controllers for synergistic robustness.

In this paper, robots are modeled as single-integrator agents equipped with relative position sensors to estimate neighbor displacements within a directed graph topology. The continuous-time OGF and OExpGF algorithms are introduced along with their time-discretized versions. Theoretical analysis establishes that OExpGF guarantees asymptotic consensus, whereas OGF ensures bounded formation distortion. 
MATLAB simulations with twelve robots in a leaderless formation under intermittent disturbances show that augmenting node-level controllers with OGF reduces the median cumulative Root Mean Distortion Error by approximately $1.4\%$–$33.14\%$, while OExpGF achieves reductions of $1.2\%$–$30.88\%$, relative to the corresponding standalone physical-layer methods. Hardware experiments on a Qualisys testbed with Crazyflie 2.0 quadrotors in a leader–follower formation further demonstrate that OGF and OExpGF lower median formation distortion by over $62\%$ and $31.4\%$, respectively, compared to fixed weights consensus. Overall, the results highlight a trade-off: OGF’s unconstrained weights deliver superior formation accuracy but may grow large, whereas OExpGF enforces constrained convex weights, enhancing long-term implementability and predictability.

\begin{figure*}[htbp]
    \centering
    \includegraphics[width=\textwidth]{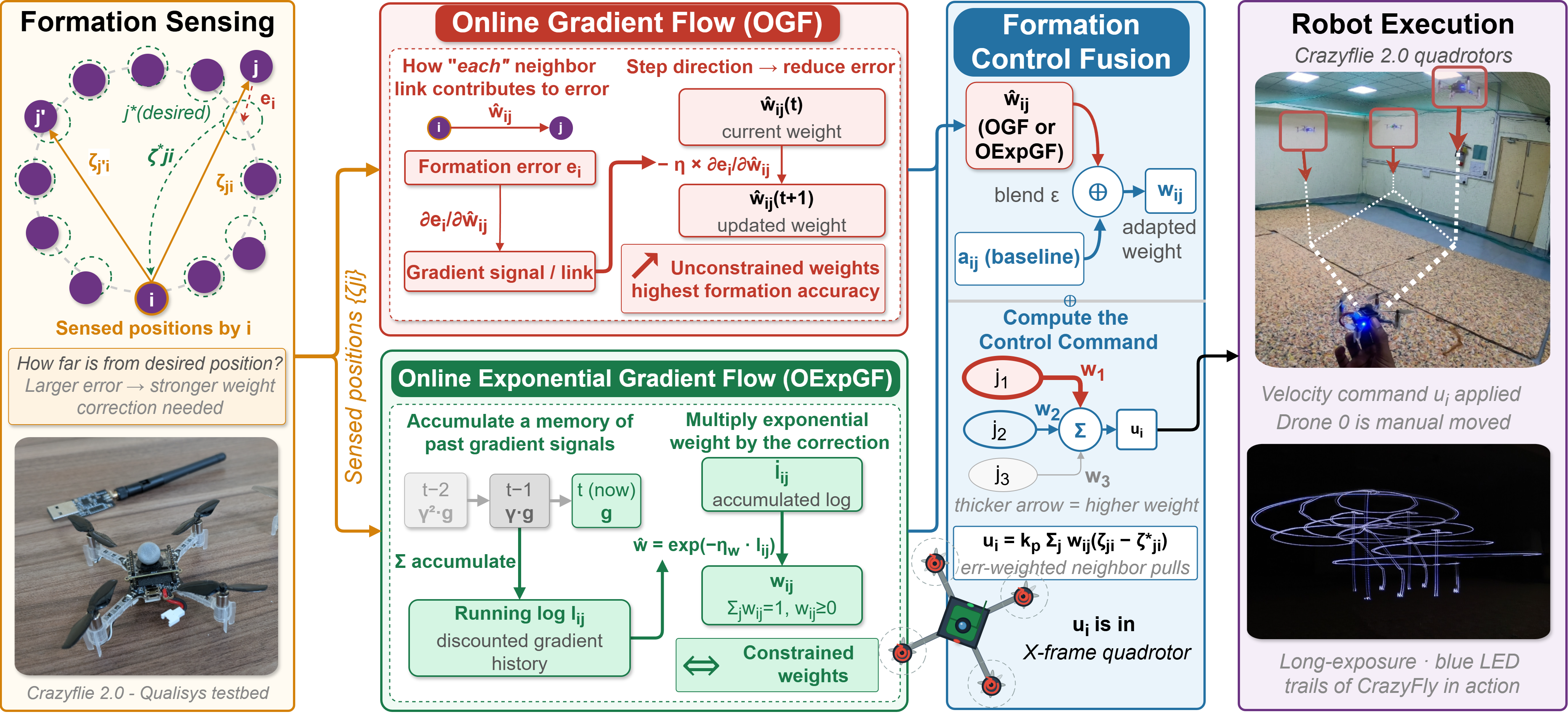}
    \caption{Overview of the TOLD formation control framework. OGF and OExpGF adapt the robot-robot interaction topology weights online from sensed relative positions; the fused weights drive the formation control law $\mathbf{u}_i$ in a closed loop.}
  \label{Fig:Overview}
\end{figure*}

\section{Problem Formulation}
The multi-robot formation control problem considered here is displacement-based, where $N$ robots aim to maintain a desired geometric formation defined by their displacements relative to neighbors. Robots interact by directly observing neighbors’ relative positions via exteroceptive sensors (e.g., camera, LiDAR, Radar). The interaction network is a directed graph $\mathcal{G}$ with row-stochastic weighted adjacency matrix $\mathbf{A} = [a_{ij}] \in \mathbb{R}^{N\times N}$, where $\sum_{j=1}^{N} a_{ij} = 1$, $a_{ij} > 0$ if an interaction link exists from robot $i$ to $j$, and $a_{ij} = 0$ otherwise, with $a_{ii} = 0$. Let $\mathcal{N}_i$ denote the neighbor set of robot $i$ in $\mathcal{G}$.

Considering interactions via direct observations through exteroceptive sensors, the $i^{th}$ robot’s estimate of the relative position of robot $j$ with respect to itself, denoted as $\zeta_{ji}$, can be modeled as 
\begin{equation} \label{e2.01}
    \zeta_{ji} = \mathbf{x}_{j} -\mathbf{x}_i
\end{equation}
where $\mathbf{x}_i \in \mathbb{R}^n$ and $\mathbf{x}_j \in \mathbb{R}^n$ are the $n$-dimensional positions ($m$) of the $i^{th}$ robot and the $j^{th}$ robot, respectively. 
Given a geometrically realizable target formation, let the target displacements for robot $i$ be $\mathbf{\zeta}_{ji}^*$, $\forall j \in \mathcal{N}_i$. The robots know target displacements only for their neighbors $j \in \mathcal{N}_i$, not for the entire multi-robot system.

Consider $N$ robots following the single integrator kinematic model, $\forall i \in [N]$:
\begin{equation} \label{e1}
    \dot{\mathbf{x}}_i = \mathbf{u}_i \; , \; \mathbf{x}_i(0)
\end{equation}

where $\mathbf{x}_{i} \in \mathbb{R}^n$ is the $n$-dimensional position ($m$), $\mathbf{u}_{i} \in \mathbb{R}^n$ is the $n$-dimensional velocity control input ($m/s$) of the $i^{th}$ robot. 
With $\mathbf{A} = [a_{ij}]$ as the row-stochastic weighted adjacency matrix of the interaction network $\mathcal{G}$, the robots employ a consensus-based formation control law as follows
\begin{equation} \label{e2}
    \mathbf{u}_i = k_p \sum_{j=1}^{N} a_{ij} (\zeta_{ji} - \mathbf{\zeta}_{ji}^*)
\end{equation} 
where $[\mathbf{A}]_{ij} = a_{ij} \geq 0$, $k_p$ is the control gain, $\zeta_{ji}^*$ is the target displacement between the $j^{th}$ robot and the $i^{th}$ robot w.r.t. the $i^{th}$ robot, $\zeta_{ji}$ is the $i^{th}$ robot's relative displacement sensor's measurement of the relative position of $j^{th}$ robot w.r.t. the $i^{th}$ robot. 

\begin{assumption}
  The target formation, as defined by the target displacements $\mathbf{\zeta}_{ji}^*$, is geometrically realizable or consistent, such 
 that one can find $\mathbf{x}_j^*, \mathbf{x}_i^* \in \mathbb{R}^n$ where $\zeta_{ji}^*=\mathbf{x}_j^*-\mathbf{x}_i^*$. 
\end{assumption} 

Define $\mathbf{y}_i := \mathbf{x}_i - \mathbf{x}_i^*$ and $\mathbf{y} := [\mathbf{y}_1^\text{T} \; \mathbf{y}_2^\text{T} \; \cdots \; \mathbf{y}_N^\text{T}]^\text{T}$, where $(\cdot)^\text{T}$ is the transpose operation. 

Using equations (\ref{e1}) and (\ref{e2}), and substituting $\zeta_{ji}^*=\mathbf{x}_j^*-\mathbf{x}_i^*$ and noting that $\mathbf{y}_i = \mathbf{x}_i - \mathbf{x}_i^*$, we get
\begin{equation} \label{e2.3}
    \dot{\mathbf{y}}_i = k_p \sum_{j=1}^{N} a_{ij} (\mathbf{y}_j - \mathbf{y}_i)
\end{equation}
Stacking up equation (\ref{e2.3}), $\forall i \in [N]$, leads to the following:
\begin{equation} \label{e2.4}
    \dot{\mathbf{y}} = - k_p (\mathbf{L} \otimes \mathbf{I}_n) \mathbf{y}
\end{equation}
where the Laplacian matrix $\mathbf{L} = \mathbf{I}_N - \mathbf{A}$. Here, $\mathbf{I}_N$ is a $N \times N$ identity matrix, whereas $\mathbf{I}_n$ is a $n \times n$ identity matrix, and $\otimes$ is the Kronecker product operator.

\begin{assumption}
    The directed graph $\mathcal{G}$ contains a rooted spanning tree, i.e., there exists at least one subgraph of $\mathcal{G}$ which contains a node (robot) called root, which has no incoming links and from which all the nodes of the graph $\mathcal{G}$ can be reached. That is, there exists at least one such node $k \in [N]$, such that all the other nodes of the graph $\mathcal{G}$ may be reached from node $k$. 
\end{assumption}

Assumption 2 is a standard requirement for consensus in directed graphs. Given that $\mathbf{A}$ is row-stochastic with non-negative entries, under Assumptions 1 and 2, it has been shown that the solution to equation (\ref{e2.4}) reaches consensus \cite{ren2008distributed}, ensuring that the target formation is achieved for robots following the single-integrator dynamics and the control law in equations (\ref{e1}) and (\ref{e2}).

\section{Topological Online Learning for Displacement-based Formation Control}     
Denote the $i^{th}$ robot's neighborhood formation distortion error as follows:
\begin{equation} \label{e12}
    e_i(t) := \frac{1}{2} \sum_{j \in \mathcal{N}_i} \| \mathbf{x}_{j} -\mathbf{x}_i - \mathbf{\zeta}_{ji}^* \|^2 
\end{equation}
where  $\|\cdot\|$ is the 2-norm or Euclidean distance operator. From equation (\ref{e2.01}), consider $\zeta_{ji} = \mathbf{x}_{j} -\mathbf{x}_i$ as the $i^{th}$ robot's exteroceptive sensor's estimate of the relative position of $j^{th}$ robot w.r.t. the $i^{th}$ robot.

Two approaches are proposed under the Topological Online Learning for Displacement-based (TOLD) formation control framework, differing in the online learning method used to update the weights in the consensus-like weighted error sum that serves as the robots’ control input. The overall framework is conceptually illustrated in Fig. \ref{Fig:Overview}.

The TOLD formation control law is given as follows:
\begin{equation} \label{e3}
    \mathbf{u}_i = k_p \sum_{j=1}^{N} w_{ij}(t) (\mathbf{\zeta}_{ji} - \mathbf{\zeta}_{ji}^*)
\end{equation}
where the weights $w_{ij}(t)$ can be updated using either of the following two approaches:

\textit{1. Online Gradient Flow (OGF):}
\begin{equation} \label{e4.11}
    w_{ij}(t) = (1-\epsilon) \, a_{ij} \, \hat{w}_{ij}(t) + \epsilon \, a_{ij} 
\end{equation}
Here, $\epsilon$ is a user-defined parameter ($0 < \epsilon \ll 1$), which ensures that the original connectivity (as per the graph $\mathcal{G}$ associated with $\mathbf{A}$) is preserved when $\hat{w}_{ij}(t) = 0$ for the graph associated with $\mathbf{W}(t)$, where $[\mathbf{W}(t)]_{ij} = w_{ij}(t)$. Note that the weights $w_{ij}(t) \in \mathbb{R}$.

The weights $\hat{w}_{ij}(t)$ are updated using the following differential equation (with $\hat{w}_{ij}(0) = 1$):
\begin{equation} \label{e5.11}
   \dot{\hat{w}}_{ij}(t) =  -\eta \; \frac{\partial e_i(t)}{\partial \hat{w}_{ij}} \;
\end{equation} 
where $\eta > 0$ is the learning rate parameter. The time-discretized formulation (Algorithm \ref{alg01}) is employed for both simulation and practical implementation.
\begin{algorithm} [t]
 \caption{: Time-Discretized Online Gradient Flow}
 \begin{algorithmic}[1] \label{alg01}
 \renewcommand{\algorithmicensure}{\textbf{Choose:}}
 \renewcommand{\algorithmicrequire}{\textbf{Initialization:}}
 \ENSURE  $\eta > 0$, $0 < \epsilon \ll 1$, $\Delta T$, $k_p$
 \REQUIRE $\hat{w}_{ij}(0) = 1$, $\frac{\partial x_i(0)}{\partial \hat{w}_{ij}} = \mathbf{0}_{n\times 1}$ \\
 \textbf{Input:} $\{\zeta_{ji}(k)\}_{j\in \mathcal{N}_i}, \hat{w}_{ij}(k), \frac{\partial x_i(k)}{\partial \hat{w}_{ij}}$ \\
 \textbf{Output:} $u_{\text{OGF}}(k), \hat{w}_{ij}(k+1), \frac{\partial x_i(k+1)}{\partial \hat{w}_{ij}}$ 
  \STATE $\nabla_{\mathbf{x}_i} e_i(k) = - \sum_{j' \in \mathcal{N}_i} (\mathbf{\zeta}_{j'i}(k) - \mathbf{\zeta}_{j'i}^*)$
  \STATE $\frac{\partial e_i(k)}{\partial \hat{w}_{ij}} = \left\langle \nabla_{\mathbf{x}_i} e_i(k) , \frac{\partial \mathbf{x}_i(k)}{\partial \hat{w}_{ij}} \right\rangle$ 
  \STATE $\hat{w}_{ij}(k+1) = \hat{w}_{ij}(k) - \Delta T \cdot \eta \cdot \frac{\partial e_i(k)}{\partial \hat{w}_{ij}}$
  \STATE $w_{ij}(k) = (1-\epsilon) a_{ij} \hat{w}_{ij}(k) + \epsilon \, a_{ij}$
  \STATE $\frac{\partial x_i(k+1)}{\partial \hat{w}_{ij}} = \frac{\partial x_i(k)}{\partial \hat{w}_{ij}} + \Delta T \cdot k_p \cdot (\mathbf{\zeta}_{ji}(k) - \mathbf{\zeta}_{ji}^*)$
  \STATE $u_{\text{OGF}}(k) = k_p \sum_{j=1}^{N} w_{ij}(k) (\mathbf{\zeta}_{ji}(k) - \mathbf{\zeta}_{ji}^*)$
 \end{algorithmic} 
\end{algorithm}

\textit{2. Online Exponential Gradient Flow (OExpGF):}
\begin{equation} \label{e4}
    w_{ij}(t) = (1-\epsilon) \frac{a_{ij} \hat{w}_{ij}(t)}{\sum_{j' \in \mathcal{N}_i} a_{ij'} \hat{w}_{ij'}(t)} + \epsilon \, a_{ij} 
\end{equation}
where $\mathcal{N}_i$ is the neighbor set corresponding to the $i^{th}$ robot, and $\epsilon$ is a user-defined parameter ($0 < \epsilon \ll 1$), which ensures that the original connectivity (as per the graph $\mathcal{G}$ associated with $\mathbf{A}$) is preserved at all times for the graph associated with $\mathbf{W}(t)$, where $[\mathbf{W}(t)]_{ij} = w_{ij}(t)$. This implies that the graph associated with $\mathbf{W}(t)$ also contains a rooted spanning tree at any time $t$. Note that $w_{ij}(t) = a_{ij}$ if $\hat{w}_{ij}(t) = \hat{w}_{ij'}(t)$, $\forall j' \in \mathcal{N}_i$. 

The weights $\hat{w}_{ij}(t)$ are updated as follows:
\begin{equation} \label{e5}
   \hat{w}_{ij}(t) =  \exp{(-\eta_w \; l_{ij}(t))} 
\end{equation}
where 
\begin{equation} \label{e6}
   l_{ij}(t) := \int_{0}^{t} \gamma^{t-t'} \; \frac{\partial e_i(t')}{\partial w_{ij}} \; dt'
\end{equation}
Here, $\gamma \in [0,1]$ is the discount factor (for discounting the past gradients), and $\|\cdot\|$ is the 2-norm or Euclidean distance operator. Note that the weights $w_{ij}(t)$ are non-negative and satisfy $\sum_{j=1}^N w_{ij}(t) = 1$, $\forall i \in [N]$.  

Further, taking derivative of equation \ref{e6} yields:
\begin{equation} \label{e7}
   \dot{l}_{ij}(t) = \log{\gamma} \; \; l_{ij}(t) + \frac{\partial e_i(t)}{\partial w_{ij}} 
\end{equation}
which is valid for $\gamma \in (0,1]$. The derivative of equation (\ref{e5}) yields:
\begin{equation} \label{e8}
   \dot{\hat{w}}_{ij}(t) =  -\eta_w \; \hat{w}_{ij}(t) \left(\log{\gamma} \; \; l_{ij}(t) + \frac{\partial e_i(t)}{\partial w_{ij}} \right)
\end{equation}
where $\eta_w > 0$ is the learning rate parameter, and $\hat{w}_{ij}(0) = 1$. 

\begin{algorithm} [t]
 \caption{: Time-Discretized Online Exponential Gradient Flow}
 \begin{algorithmic}[1] \label{alg02}
 \renewcommand{\algorithmicensure}{\textbf{Choose:}}
 \renewcommand{\algorithmicrequire}{\textbf{Initialization:}}
 \ENSURE  $\eta_w > 0$, $0 <\gamma \leq 1$, $0 < \epsilon \ll 1$, $\Delta T$, $k_p$
 \REQUIRE $\hat{w}_{ij}(0) = 1$, $\frac{\partial x_i(0)}{\partial w_{ij}} = \mathbf{0}_{n\times 1}$, $l_{ij}(0) = 0$ \\
 \textbf{Input:} $\{\zeta_{ji}(k)\}_{j\in \mathcal{N}_i}, \hat{w}_{ij}(k), \frac{\partial x_i(k)}{\partial w_{ij}}$, $l_{ij}(k)$ \\
 \textbf{Output:} $u_{\text{OExpGF}}(k), \hat{w}_{ij}(k+1), \frac{\partial x_i(k+1)}{\partial w_{ij}}, l_{ij}(k+1)$ 
  \STATE $\nabla_{\mathbf{x}_i} e_i(k) = - \sum_{j' \in \mathcal{N}_i} (\mathbf{\zeta}_{j'i}(k) - \mathbf{\zeta}_{j'i}^*)$
  \STATE $\frac{\partial e_i(k)}{\partial w_{ij}} = \left\langle \nabla_{\mathbf{x}_i} e_i(k) , \frac{\partial \mathbf{x}_i(k)}{\partial w_{ij}} \right\rangle$ 
  \STATE ${l}_{ij}(k+1) = {l}_{ij}(k) + \Delta T  (\log{\gamma} \; \; l_{ij}(k) + \frac{\partial e_i(k)}{\partial w_{ij}})$
  \STATE $\hat{w}_{ij}(k+1) = \hat{w}_{ij}(k) - \Delta T \eta_w \hat{w}_{ij}(k) (\log{\gamma} \; l_{ij}(k) + \frac{\partial e_i(k)}{\partial w_{ij}})$
  \STATE $w_{ij}(k) = (1-\epsilon) \frac{a_{ij} \hat{w}_{ij}(k)}{\sum_{j' \in \mathcal{N}_i} a_{ij'} \hat{w}_{ij'}(k)} + \epsilon \, a_{ij}$
  \STATE $\frac{\partial x_i(k+1)}{\partial w_{ij}} = \frac{\partial x_i(k)}{\partial w_{ij}} + \Delta T \cdot k_p \cdot (\mathbf{\zeta}_{ji}(k) - \mathbf{\zeta}_{ji}^*)$
  \STATE $u_{\text{OExpGF}}(k) = k_p \sum_{j=1}^{N} w_{ij}(k) (\mathbf{\zeta}_{ji}(k) - \mathbf{\zeta}_{ji}^*)$
 \end{algorithmic} 
\end{algorithm}
For both OGF (replace $w_{ij}$ with $\hat{w}_{ij}$ in the following) and OExpGF, the gradient $\frac{\partial e_i(t)}{\partial w_{ij}}$ is calculated as follows:
\begin{equation} \label{e9}
    \frac{\partial e_i(t)}{\partial w_{ij}} = \left\langle \nabla_{\mathbf{x}_i} e_i(t) , \frac{\partial \mathbf{x}_i(t)}{\partial w_{ij}} \right\rangle 
\end{equation}
\begin{equation} \label{e10}
     \nabla_{\mathbf{x}_i} e_i(t) = - \sum_{j' \in \mathcal{N}_i} (\mathbf{\zeta}_{j'i} - \mathbf{\zeta}_{j'i}^*) 
\end{equation}
and
\begin{equation} \label{e11}
    \frac{\partial \mathbf{x}_i(t)}{\partial w_{ij}} = \int_{0}^{t} k_p \;(\mathbf{\zeta}_{ji} - \mathbf{\zeta}_{ji}^*) \; dt'
\end{equation} 
Here $\langle\cdot,\cdot\rangle$ is the dot product operator. The time-discretized formulation (Algorithm \ref{alg02}) is employed for both simulation and practical implementation. 
During OExpGF's implementation, the gradients $\frac{\partial e_i(t)}{\partial w_{ij}}$ given by equation (\ref{e9}) are normalized by dividing them by $\max\{1,\max_{j\in\mathcal{N}_i}|{\frac{\partial e_i(t)}{\partial w_{ij}}}|\}$ to avoid numerical instability. 

\section{Theoretical Analysis}
\subsection{Online Exponential Gradient Flow (OExpGF)}
Use equations (\ref{e1}) and (\ref{e3}), and substitute $\zeta_{ji}^*=\mathbf{x}_j^*-\mathbf{x}_i^*$ (from Assumption 1). Further, considering $\mathbf{y}_i = \mathbf{x}_i - \mathbf{x}_i^*$ and noting that $\sum_{j=1}^N w_{ij}(t) = 1$, we get
\begin{equation} \label{e15}
    \dot{\mathbf{y}}_i = k_p \sum_{j=1}^{N} w_{ij}(t) (\mathbf{y}_j - \mathbf{y}_i)
\end{equation}
Stacking up equation (\ref{e15}), $\forall i \in [N]$, leads to the following:
\begin{equation} \label{e16}
    \dot{\mathbf{y}} = - k_p (\mathcal{L}(t) \otimes \mathbf{I}_n) \mathbf{y}
\end{equation}
where the Laplacian matrix $\mathcal{L}(t) = \mathbf{I}_N - \mathbf{W}(t)$, $[\mathbf{W}(t)]_{ij} = w_{ij}(t)$, and $\mathbf{y} := [\mathbf{y}_1^\text{T} \; \mathbf{y}_2^\text{T} \; \cdots \; \mathbf{y}_N^\text{T}]^\text{T}$. Here, $\mathbf{W}(t)$ is row-stochastic by design and has non-negative entries (check equations (\ref{e4}) and (\ref{e5})). 
\begin{theorem}
\label{thm:consensus}
Under Assumptions 1 and 2, consider $N$ robots with their states $\mathbf{y}_i(t) \in \mathbb{R}^n$, $\forall i \in [N]$, evolving under eq. (\ref{e15}), where $k_p > 0$, and the weight matrix $\mathbf{W}(t)$, with $[\mathbf{W}(t)]_{ij} = w_{ij}(t)$, updated using the OExpGF algorithm (eq. (\ref{e4})--(\ref{e11})). 
Then, for any initial condition $\mathbf{y}(0) \in \mathbb{R}^{Nn}$, asymptotic consensus is achieved:
\begin{equation}
\lim_{t \to \infty} \|\mathbf{y}_i(t) - \mathbf{y}_j(t)\| = 0 \quad \forall i,j \in [N]
\end{equation}
where $\mathbf{y}_i(t) \equiv \mathbf{y}_i$ and $\mathbf{y}(t) \equiv \mathbf{y}$ from eq. (\ref{e15}) and eq. (\ref{e16}). 
\end{theorem}
\begin{proof}
Define the convex hull $\mathcal{C}(t)$ and its diameter $ D(t)$ as follows:
\begin{equation}
    \mathcal{C}(t) := \operatorname{conv}\{\mathbf{y}_1(t), \ldots, \mathbf{y}_N(t)\}
\end{equation}
\begin{equation}
    D(t) := \max_{i,j \in [N]} \|\mathbf{y}_i(t) - \mathbf{y}_j(t)\|
\end{equation}
From eq. (\ref{e15}), the $i^{th}$ robot's velocity $\dot{\mathbf{y}}_i(t)$ can be written as:
\begin{equation}
\dot{\mathbf{y}}_i(t) = k_p \left(\sum_{j=1}^N w_{ij}(t)\mathbf{y}_j(t) - \mathbf{y}_i(t)\right)
\end{equation}
Since $w_{ij}(t) \geq 0$ and $\sum_{j=1}^{N} w_{ij}(t) = 1$, the term $\sum_{j=1}^{N} w_{ij}(t)\mathbf{y}_j(t)$ is a convex combination of robots' states $\mathbf{y}_j(t)$, $\forall j \in [N]$, at time $t$. Thus, the $i^{th}$ robot's state $\mathbf{y}_i(t)$ moves toward a point inside the convex hull $\mathcal{C}(t)$. By standard results on differential inclusions in convex sets \cite[Lemma 1]{moreau2005stability}, the convex hull is forward-invariant:
\begin{equation} \label{e_bound}
\mathcal{C}(t) \subseteq \mathcal{C}(0) \quad \forall t \geq 0
\end{equation}
and therefore, $D(t)$ is bounded.

For any $t \geq 0$, since each robot's state $\mathbf{y}_i(t)$ remains within the convex hull of all robots' states (which is itself contained in the initial hull $\mathcal{C}(0)$), no two robots' states can move apart to create a larger pairwise distance than already exists. Formally, by the forward invariance of $\mathcal{C}(t)$ and the definition of diameter $D(t)$ as the maximum distance within this set \cite[Lemma 1]{moreau2005stability}:
\begin{equation} \label{e_non_inc}
D(t) \leq D(s) \quad \forall t \geq s \geq 0
\end{equation}
Thus, $D(t)$ is non-increasing in time.

From eq. (\ref{e4}) and (\ref{e5}), note that the weights $w_{ij}(t)$ are uniformly lower-bounded as $w_{ij}(t) \ge \epsilon a_{ij} > 0$. Under Assumption 2, the presence of a rooted spanning tree in the graph associated with $\mathbf{A}$ guarantees that the time-varying graph induced by $\mathbf{W}(t)$ retains a rooted spanning tree for all $t \geq 0$. By \cite[Theorem 2]{moreau2005stability}, under persistent rooted spanning tree connectivity, if $D(t_0) > 0$, then there exist $T, \delta > 0$ such that:
\begin{equation} \label{e_contract}
D(t_0 + T) \leq (1 - \delta)D(t_0) 
\end{equation}
This exponential contraction follows from information propagation through the spanning tree: the diameter strictly decreases over bounded intervals as long as disagreement ($D(t_0) > 0$) persists. 

Since $D(t)$ is non-increasing (eq. (\ref{e_non_inc})), bounded (eq. (\ref{e_bound})), and strictly contracts when positive (eq. (\ref{e_contract})), we have:
\begin{equation}
\lim_{t \to \infty} D(t) = 0
\end{equation}
Therefore, $\lim_{t \to \infty} \|\mathbf{y}_i(t) - \mathbf{y}_j(t)\| = 0$ for all $i,j \in [N]$, establishing asymptotic consensus.
\end{proof}
\subsection{Online Gradient Flow (OGF)}
\begin{theorem}
    Under Assumption 1, when the OGF algorithm (eq. (\ref{e3}), (\ref{e4.11}), (\ref{e5.11})) is run on the single integrator kinematics (eq. (\ref{e1})), $\forall i \in [N]$, the system's formation distortion error $e_i(t)$ (eq. (\ref{e12})) does not grow with time. 
\end{theorem}
\begin{proof}
    Consider the error $e_i(t)$ and the weight $w_{ij}(t)$ from eq. (\ref{e4.11}). Note that $\frac{de_i(t)}{dt} = \sum_{\forall j \in \mathcal{N}_i}\frac{\partial {e}_i(t)}{\partial {w}_{ij}}\frac{dw_{ij}(t)}{dt}$. Further, using eq. (\ref{e4.11}), note that $\frac{dw_{ij}(t)}{dt} = (1-\epsilon) a_{ij} \frac{d\hat{w}_{ij}(t)}{dt}$ and $\frac{\partial {e}_i(t)}{\partial {w}_{ij}} = ((1-\epsilon)a_{ij})^{-1} \frac{\partial {e}_i(t)}{\partial \hat{w}_{ij}}$. Using eq. (\ref{e5.11}), this implies $\frac{de_i(t)}{dt} = - \eta \sum_{\forall j \in \mathcal{N}_i} (\frac{\partial {e}_i(t)}{\partial \hat{w}_{ij}})^2 \leq 0$. 
\end{proof}

\section{Performance Evaluation}
The continuous-time OGF and OExpGF are time-discretized (Algorithms \ref{alg01} and \ref{alg02}) and implemented alongside a continuous-time single integrator plant simulated in MATLAB using the ode45 solver, with a sampling period $\Delta T$ and a time horizon $T$. Their performance is evaluated under a 2-D scenario where robots execute constant-velocity commands within a decentralized, leaderless formation while being subjected to intermittent disturbances (e.g., wind gusts).  

Consider the plant dynamics (eq. (\ref{e1})) affected by the disturbance $\mathbf{\mu}_{i}$ as follows: 
\begin{equation}
     \dot{\mathbf{x}}_i = \mathbf{u}_i + \mathbf{\mu}_{i}\; , \; \mathbf{x}_i(0)  
\end{equation}
where $\mathbf{\mu}_{i}  \in \mathbb{R}^2$ models the intermittent disturbances. Here, the control input $\mathbf{u}_i$ is given as follows: 
\begin{equation} \label{eq:u_i}
    \mathbf{u}_i = \mathbf{v}_{\text{const.},i} + \mathbf{u}_{\text{form.},i} 
\end{equation}
where $\mathbf{v}_{\text{const.},i} \in \mathbb{R}^2$ is the $i^{th}$ robot's commanded constant velocity and $\mathbf{u}_{\text{form.},i}$ is the formation control command. For this simulation study, $\mathbf{v}_{\text{const.},i} = \frac{1}{\sqrt{2}}[1 \; 1]^T$ $m/s$. 
Further, consider a saturation limit on the control input due to limited actuator control authority, resulting in a bound on $\mathbf{u}_i = [u_{i,1} \; u_{i,2}]^T$ as $|u_{i,k}| \leq u_{\max} > 0$, $k=1,2$, where $(\cdot)^T$ denotes the transpose.

Here, $\mathbf{u}_{\text{form.},i}$ may correspond to fixed-weights consensus ($\mathbf{u}_{\text{Fxd.Wts.},i}$, eq. \eqref{e2}), pure edge-level adaptation ($\mathbf{u}_{\text{edge},i}$), pure node-level adaptation ($\mathbf{u}_{\text{node},i}$), or their combination ($\mathbf{u}_{\text{edge},i} + \mathbf{u}_{\text{node},i}$). The edge-level term $\mathbf{u}_{\text{edge},i}$ is computed via OGF or OExpGF (eq. \eqref{e3}), while the node-level term $\mathbf{u}_{\text{node},i}$ is realized using a robust controller such as Adaptive Gain, Decay Gain, or an anti-windup Disturbance Observer (DOB) \cite{r1}, \cite{r2}, \cite{r4}, as follows (with $t = k \ \Delta T$):

\textbf{Adaptive Gain:} The control $\mathbf{u}_{\text{node},i}$ is computed as \cite{r1}:
\begin{equation}
    \mathbf{u}_{\text{node},i}(k) =  k_p \sum_{j = 1}^N \left(1 + \beta_{i}(k)\right) a_{ij} (\mathbf{\zeta}_{ji}(k) - \mathbf{\zeta}_{ji}^*)
\end{equation}
where the adaptive gain $\beta_{i}(k)$ is updated as:
\begin{equation}\begin{array}{rr}
    \beta_{i}(k + 1) =  & \max \{ 0,  \beta_{i}(k) + \Delta T  (\sigma  \|\nabla_{\mathbf{x}_i} e_i(k)\|^2 \\
     & - \kappa  \beta_{i}(k))\}
\end{array}
\end{equation}
with the adaptation rate $\sigma = 0.5$ and the leakage (forgetting) factor $\kappa = 0.1$. 

\textbf{Decay Gain:} The control $\mathbf{u}_{\text{node},i}$ is computed as \cite{r2}:
\begin{equation}
    \mathbf{u}_{\text{node},i}(k) =  k_p \sum_{j =1}^N \frac{1}{(1 + k \ \Delta T)^{\alpha}} a_{ij} (\mathbf{\zeta}_{ji}(k) - \mathbf{\zeta}_{ji}^*)
\end{equation}
where the decay exponent $\alpha = 0.6$. 

\textbf{Disturbance Observer:} The control $\mathbf{u}_{\text{node},i}$ is computed as \cite{r4}:
\begin{equation}
    \mathbf{u}_{\text{node},i}(k) = k_p \sum_{j=1}^N  a_{ij} (\mathbf{\zeta}_{ji}(k) - \mathbf{\zeta}_{ji}^*) - \text{sat}_{u_{\max}} \left( \hat{\mathbf{d}}_i(k) \right)
\end{equation}
where $\hat{\mathbf{d}}_i(k)$ is the estimated disturbance and $\text{sat}_{u_{\max}}(\cdot)$ denotes the element-wise saturation function to enforce actuator limits. The estimated disturbance $\hat{\mathbf{d}}_i(k)$ is calculated as: 
\begin{equation}
    \hat{\mathbf{d}}_i(k) = \mathbf{\xi}_i(k) + \Lambda \ \mathbf{x}_i(k) 
\end{equation}
where $\Lambda = 10 \cdot \text{diag}([1 \ 1])$ is the observer gain matrix, and the auxiliary state vector $\mathbf{\xi}_i(k)$ is updated as: 
\begin{equation}
    {\mathbf{\xi}}_i(k+1) = {\mathbf{\xi}}_i(k) - \Delta T \ \Lambda \left( \mathbf{\xi}_i(k) + \Lambda \mathbf{x}_i(k) + \mathbf{u}_{i}(k) \right)
\end{equation}
where $\mathbf{u}_{i}(k)$ is given by eq. (\ref{eq:u_i}). 

\begin{figure}[t]
    \centering
    \includegraphics[width=0.48\textwidth]{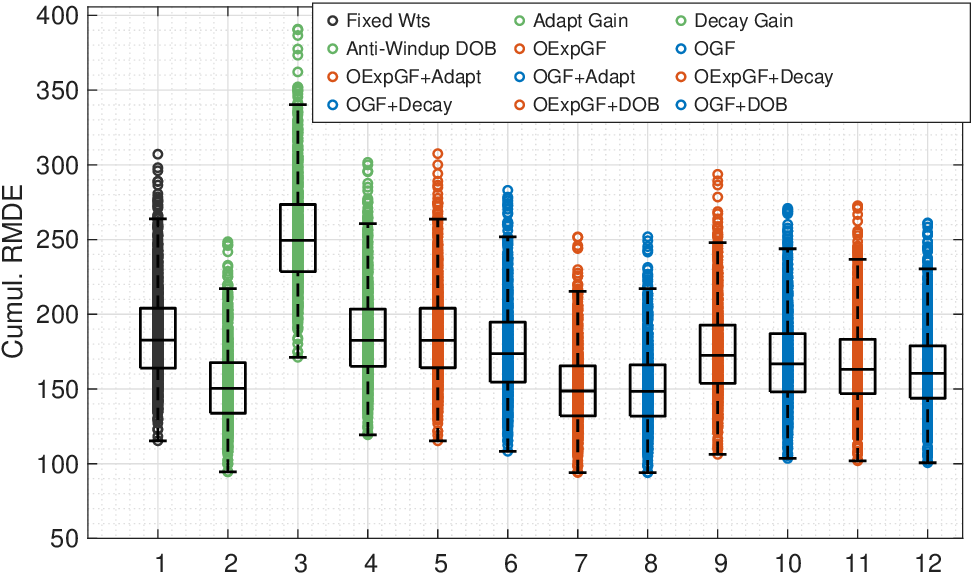}
    \caption{Cumulative Root Mean Distortion Error (RMDE) (in $m$) for 1000 simulation runs; total no. of robots, N $=12$; no. of robots experiencing disturbances, Nf $=6$; 1: Fixed Wts, 2: Adaptive Gain, 3: Decay Gain, 4: Anti-Windup Disturbance Observer (DOB), 5: OExpGF, 6: OGF, 7: OExpGF + Adaptive Gain, 8: OGF + Adaptive Gain, 9: OExpGF + Decay Gain, 10: OGF + Decay Gain, 11: OExpGF + DOB, 12: OGF + DOB.}
  \label{Fig:PLOT}
\end{figure}
The total number of robots is $N = 12$, with an actuator saturation limit of $u_{\text{max}} = 2$ m/s. The robot-robot interaction network is modeled as a strongly connected directed circulant graph $C_{12}(1, 2)$, where the adjacency matrix $A$ is defined such that each robot $i$ observes two neighbors $j$ in a unidirectional sequence $j \in \{(i \text{ mod } N) + 1, (i+1 \text{ mod } N) + 1\}$ with edge weights $a_{ij} = 0.5$, and $a_{ij} = 0$ otherwise. This structure maintains a $2$-regular directed topology, and the choice of weights $a_{ij}$ reflects the assumption that all robots are equally functional prior to operation. The target displacements $\mathbf{\zeta}_{ji}^* \in \mathbb{R}^2$ are defined to ensure that the robots maintain a two-dimensional dodecagonal formation of radius $10$ ${m}$, consistent with the underlying directed topology. The time horizon is $T = 9$ $sec.$ with a sampling period $\Delta T$ of $0.1$ $sec$. The initial weights $\hat{w}_{ij}(0)$ for both OExpGF and OGF are set to $1$, and the initial robot positions correspond to the desired formation.
\begin{remark}
    The directed circulant topology $C_{12}(1,2)$ reflects the sensing constraints of robots, where each robot monitors only two neighbors within a forward-facing sector consistent with the limited field-of-view (FOV) of sensors such as depth cameras or LiDAR \cite{moshtagh2007distributed}. From a control perspective, this topology provides a rigorous evaluation of the controller, as its Laplacian has complex eigenvalues that typically produce stronger transient oscillations and slower convergence than bidirectional topologies \cite{olfati2004consensus}.
\end{remark}

Half of the robots ($N_f = 6$) are subjected to disturbances over the interval $t \in [1,8]$ sec. These disturbances are applied intermittently at $t = 1$, $2$, $4$, and $6$ sec. to randomly selected robots, thereby emulating intermittent perturbations. The disturbance $\mathbf{\mu}_{i}$ on affected robots is generated as $1 \cdot [Unif(-1,1); Unif(-1,1)]^T + [\mathcal{N}(0,100); \mathcal{N}(0,100)]^T$ (in $m/s$), where $Unif(-1,1) \in [-1,1]$ is a uniform distribution and $\mathcal{N}(0,100)$ a normal distribution with zero mean and variance $100$. Robots not affected by disturbances are subject to process noise $[\mathcal{N}(0,0.01); \mathcal{N}(0,0.01)]^T$ (in $m/s$). Sensor noise in the relative displacement estimates $\zeta_{ji}$ is $[\mathcal{N}(0,0.01); \mathcal{N}(0,0.01)]^T$ (in $m$). Here, $(\cdot)^T$ denotes the transpose.

\begin{figure*}[htbp]
    \centering
    \includegraphics[width=\textwidth]{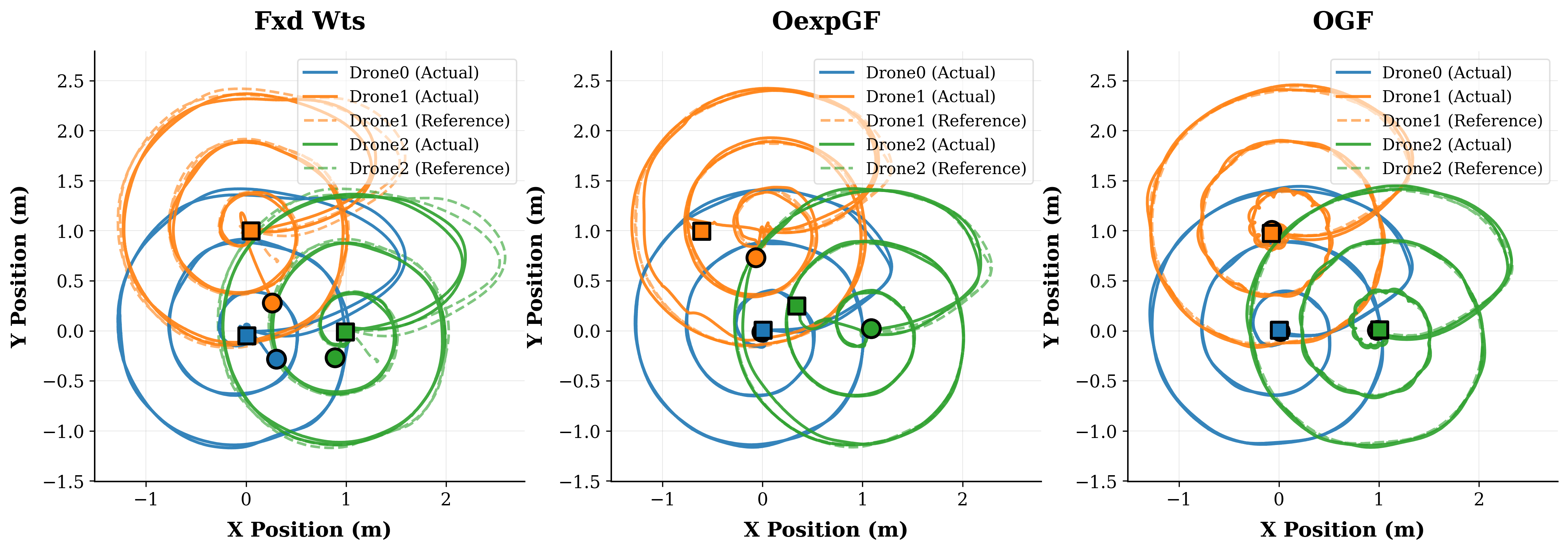}
    \caption{Top view of actual drone trajectories (solid) and velocity-based references (dashed); references are computed from Drone 0’s measured velocity and formation specifications. Circles and squares denote start and end points.}
  \label{Fig:TrajVel}
\end{figure*}

The algorithms are evaluated using the Cumulative Root Mean Distortion Error (Cumul. RMDE) at time $T$, defined as $\sum_{t=0:0.1:T} \sqrt{\frac{1}{N} \sum_{i=1}^N 2e_{i}(t)}$ (check eq. (\ref{e12})). 

Based on a simulation-based parametric study, the parameters of OExpGF and OGF are set as $\eta_w = 2$, $\gamma = 0.01$, $\epsilon = 0.01$, and $\eta = 0.1$, $\epsilon = 0.01$, respectively. The control gain $k_p = 1$ for all the algorithms.
For each method, $1000$ simulations are conducted, and the results are collected as points to generate the corresponding box plots (see Fig. \ref{Fig:PLOT}). On each box, the central mark denotes the median, and the bottom and top edges represent the $25^{th}$ and $75^{th}$ percentiles, respectively.   

From Fig. \ref{Fig:PLOT}, the Fixed-Weights consensus, OExpGF, and OGF yield cumulative RMDE (median) values of $182.68$, $182.52$, and $173.72$ $m$, respectively. Among standalone node-level methods, Adaptive Gain achieves $150.4$ $m$, Decay Gain $249.5$ $m$, and DOB $182.55$ $m$.

When edge-level adaptation is incorporated, consistent improvements over the physical-layer counterparts are observed. Specifically, OExpGF + Adaptive Gain and OGF + Adaptive Gain reduce cumulative RMDE (median) to $148.65$ $m$ and $148.36$ $m$, corresponding to improvements of approximately $1.2\%$ and $1.4\%$ over the Adaptive Gain method, respectively. 
Relative to Decay Gain, OExpGF + Decay Gain and OGF + Decay Gain reduce cumulative RMDE (median) to $172.46$ $m$ and $166.81$ $m$, yielding substantial improvements of approximately $30.88\%$ and $33.14\%$, respectively.
Similarly, compared to the standalone DOB method, OExpGF + DOB and OGF + DOB achieve cumulative RMDE (median) of $163.18$ $m$ and $160.46$ $m$, corresponding to improvements of approximately $10.61\%$ and $12.1\%$.

These results demonstrate that combining physical-layer disturbance rejection ($\mathbf{u}_{\text{node},i}$) with topological (edge-level) adaptation ($\mathbf{u}_{\text{edge},i}$) synergistically improves robustness. The approach achieves gains of up to $33.14\%$ over the low-performing Decay Gain and still reduces the median cumulative RMDE by $1.2\%$ when augmenting the high-performing Adaptive Gain, highlighting the complementary benefits of edge-level adaptation even near physical-layer performance limits.

\section{Experimental Validation}
While the previous section demonstrates that integrating TOLD with physical-layer control yields peak performance, the hardware experiments deliberately omit these node-level robust control augmentations. This isolates the core theoretical contribution, demonstrating the raw efficacy of the proposed edge-level topological adaptation against unmodeled real-world aerodynamics and sensor noise.

\subsection{Experimental Setup}
The algorithms are validated on a Qualisys motion-capture testbed using Crazyflie 2.0 quadcopters. The primary evaluation uses three drones, with additional tests on a four-drone setup. 
To evaluate leader–follower formation performance, the leader (Drone 0) is manually piloted along a predefined upward spiral trajectory, a challenging 3D path with continuously varying orientation. Each follower receives relative position measurements of its two spatial neighbors from the motion-capture system, while the leader receives no neighbor information, resulting in a directed interaction topology. The followers execute the formation algorithms using off-board velocity commands transmitted to the quadcopters, relying solely on neighbor-relative position measurements consistent with displacement-based formation control. 

\subsection{Comparative Performance Analysis}
Figure \ref{Fig:TrajVel} overlays the top-view trajectories of all three drones under the Fixed Weight, OExpGF, and OGF controllers, comparing actual (solid) against velocity-based reference (dashed) trajectories generated from Drone 0's measured velocity and the desired formation specifications. OGF (green) yields the tightest trajectory clustering, with follower paths closely shadowing the leader throughout the spiral. OExpGF (red) improves substantially over Fixed Weight, reducing inter-drone separation during rapid directional changes. In the Fixed Weights case (blue), the leader (Drone~0) overshoots during aggressive spiral segments in the upper-right region, yet the followers fail to track this deviation, resulting in significant formation divergence. This inability of fixed weights to compensate for sudden leader excursions highlights the fundamental limitation of non-adaptive consensus weights, further motivating the real-time weight adaptation provided by OGF and OExpGF. Quantitatively, relative to Fixed Weight, OGF reduces median velocity-reference tracking errors by $53.88\%$ (Drone 1) and $54.85\%$ (Drone 2), and mean errors by $40.65\%$ and $41.35\%$, with reduced interquartile ranges ($13.97\%$ and $19.28\%$) indicating improved consistency throughout the spiral maneuver. These qualitative and quantitative trends confirm that OGF's unconstrained weight adaptation most effectively minimizes distortion error.

\subsection{Formation Distortion Error}
Figure \ref{Fig:distortionBox} shows the component-wise and magnitude distributions of formation distortion errors across all three algorithms. The magnitude subplot indicates that OGF attains the lowest median error ($\approx 0.094$ $m$) with the tightest interquartile range, consistent with the MATLAB results in Section IV. OExpGF provides intermediate performance with a median error of $\approx 0.170$ $m$, while Fixed Weight exhibits the largest median error ($\approx 0.248$ $m$) and highest variance. Across individual spatial components, OGF consistently achieves the most concentrated error distributions centered near zero, whereas Fixed Weight shows the widest spread in all three axes. These per-axis results confirm that OGF's superior magnitude performance is not driven by a single spatial dimension but reflects uniformly tighter formation maintenance across all directions of motion.

\subsection{Parametric Study}
To evaluate algorithm hyperparameters, we conducted parametric studies using a fixed control gain of $ k_p = 1.8$; full numerical results are provided in the supplementary video. For OExpGF, higher learning rates ($\eta_w$) and lower discount factors ($\rho$) result in reduced mean and median errors, with optimal performance achieved at $\eta_w=5$ and $\rho=0.1$. For OGF, increasing the learning rate ($\eta$) consistently improves all metrics, with $\eta=0.01$ yielding errors $46$-$60\%$ lower than $\eta=0.0001$. Smaller gradient-integration windows ($\text{win}=1$) perform best at high learning rates, while window-size effects are minor at lower rates. Based on these results, subsequent experiments use $\eta_w=8$ and $\rho=0.5$ for OExpGF, and $\eta=0.01$ with $\text{win}=1$ for OGF. Four-drone experiments (shown in the supplementary) confirmed consistent behavior across all algorithms, indicating that the observed performance characteristics scale appropriately within the tested system size.

\begin{figure}[!htb]
    \centering
    \includegraphics[width=0.5\textwidth]{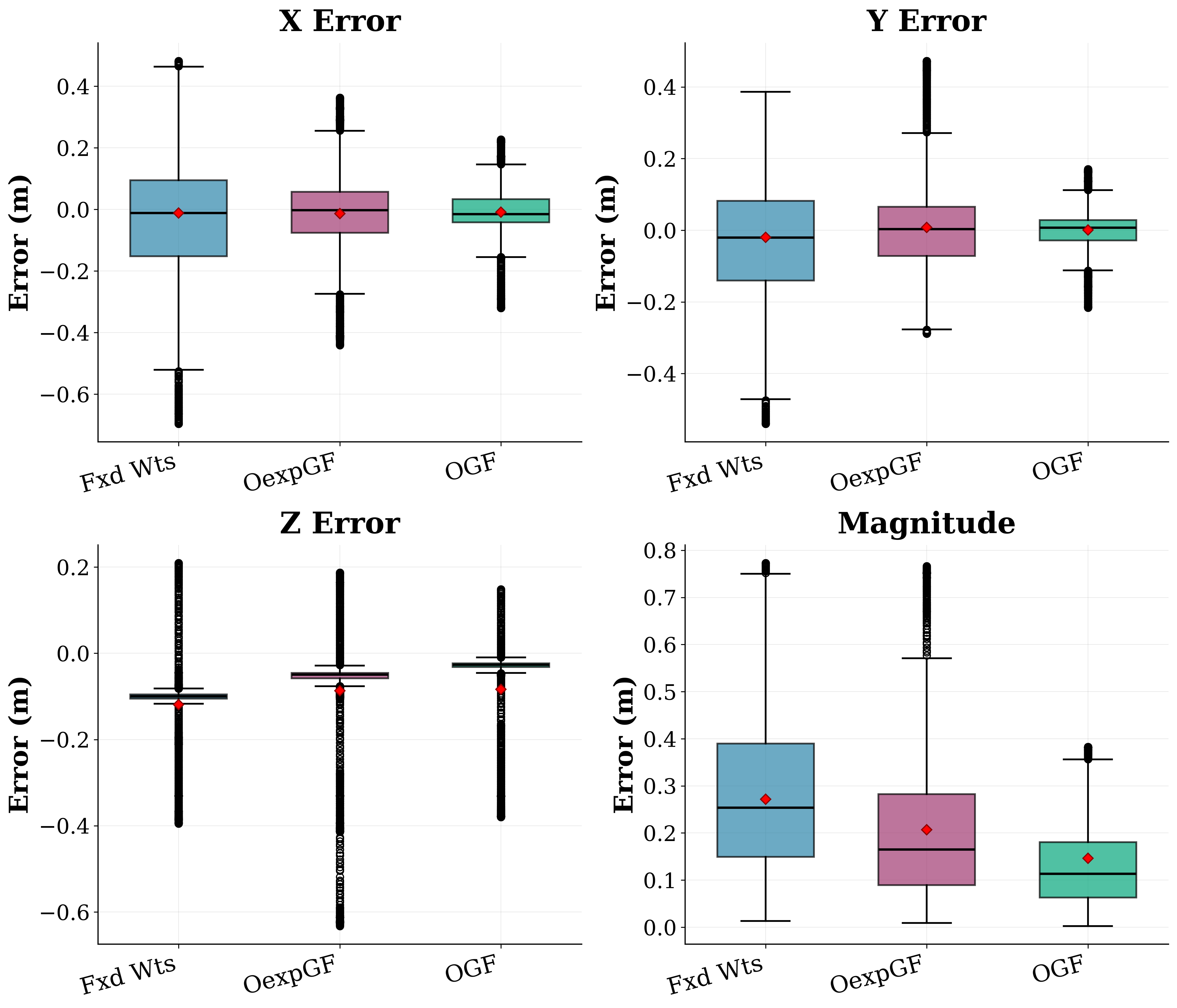}
    \caption{Component and magnitude distributions of formation distortion error across the experimental duration, aggregated over follower drones. The OGF algorithm consistently yields the lowest median error and tightest interquartile range across all dimensions, demonstrating superior spatial coordination compared to OExpGF and the Fixed-Weights baseline.}
    \label{Fig:distortionBox}
\end{figure}

\subsection{Weight Evolution}

\begin{figure}[!htb]
    \includegraphics[width=0.48\textwidth]{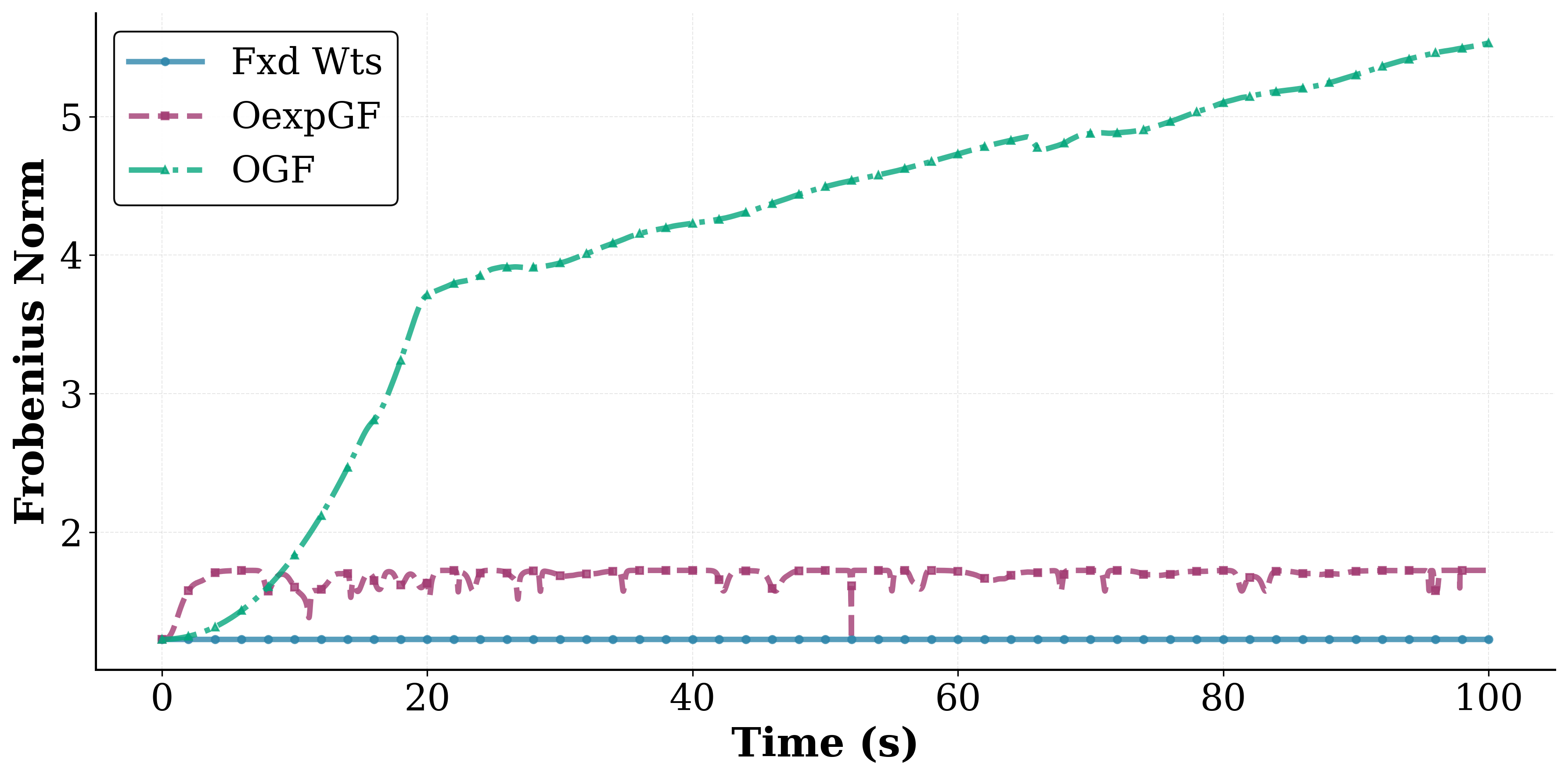}
    \caption{Temporal evolution of weight matrix Frobenius norms over a 100-second experimental run. }
  \label{Fig:forb}
\end{figure}

A key observation from Figure \ref{Fig:forb} is the evolution of the Frobenius norm of the weight matrix. OGF shows large weight growth, with the norm rising from roughly $1.2$ at initialization to over $7.0$ by the end of the run, whereas Fixed Weight and OExpGF remain bounded below $2.0$ throughout. This behavior results from OGF’s unconstrained update rule (eq. (\ref{e4.11}) and (\ref{e5.11})) where $w_{ij}(t)\in\mathbb{R}$ can grow to larger values as the algorithm minimizes (noisy) neighborhood distortion. Although this aggressive adaptation improves formation accuracy, it raises two practical concerns: (i) large weights may cause the formation objective to overpower secondary objectives (e.g., trajectory tracking or energy minimization), disrupting multi-objective balance; and (ii) on embedded platforms, growing weight magnitudes increase numerical-precision demands, risking memory overflow or instability during long-duration deployments.

Although OGF delivers the best formation accuracy as demonstrated in the preceding subsections, the large weight growth introduces implementation concerns not evident in the simulations. The finite simulation horizon in the Performance Evaluation section may simply have been too short to expose this behavior. In long-duration missions or hierarchical control settings, such weight escalation may necessitate safeguards such as weight clipping or adaptive learning-rate scheduling. In contrast, OExpGF's non-negative convex weights remain naturally bounded through the normalization in eq. (\ref{e4}), offering a practical balance between adaptability and stability. This constraint, however, limits its ability to fully match the adaptability of OGF's unconstrained weights, yielding intermediate performance; nevertheless, OExpGF can be combined with more explicit disturbance-handling strategies (robust node-level control augmentations) to achieve improved overall performance.

\section{CONCLUSION}
This paper introduces Topological Online Learning for Displacement-based (TOLD) formation control, an edge-level adaptation framework that actively minimizes formation distortion in real-time. By updating interaction weights online, TOLD complements conventional node-level robust controllers and provides a lightweight, training-free alternative to learning-based topological methods. Two strategies are proposed under the TOLD formation control framework: OGF, with unconstrained weight dynamics, and OExpGF, enforcing non-negative convex weights. Theoretical analysis establishes asymptotic consensus for OExpGF and bounded formation distortion for OGF. Simulations with twelve robots under intermittent disturbances demonstrated $1.2\%$–$33.14\%$ reduction in median cumulative RMDE when augmenting node-level controllers with TOLD. Hardware experiments with Crazyflie 2.0 quadrotors further showed over $62\%$ (OGF) and $31.4\%$ (OExpGF) reduction in median formation distortion compared to fixed-weight consensus. Overall, the results reveal a practical trade-off between OGF’s higher accuracy and OExpGF’s bounded, implementation-friendly weight dynamics, positioning TOLD as an effective and modular topological overlay for robust multi-robot formation control.





\bibliographystyle{unsrt}
\bibliography{IEEEabrv}

\end{document}